\documentclass{l4dc2025_arxiv}

\usepackage{verbatim}
\usepackage{hyperref}
\usepackage{url}
\usepackage{mathtools}

\usepackage{graphicx}
\usepackage{amssymb}
\usepackage{mathtools}

\usepackage{algorithm}
\usepackage{algorithmic}
\usepackage{booktabs}
\newtheorem{assumption}[theorem]{Assumption}

\newcommand{\norm}[1]{\lVert#1\rVert}

\newcommand{\indicator}{\mathbf{1}}

\title[Insights into Soft-Label Training]{A Theoretical Analysis of Soft-Label vs Hard-Label Training in Neural Networks}
\usepackage{times}

\author{%
 \Name{Saptarshi Mandal} \Email{smandal4@illinois.edu}\\
 \addr UIUC
 \AND
 \Name{Xiaojun Lin} \Email{xjlin@ie.cuhk.edu.hk}\\
 \addr CUHK%
 \AND
 \Name{R. Srikant} \Email{rsrikant@illinois.edu}\\
 \addr UIUC%
}

\begin{document}

\maketitle

\begin{abstract}%
 Knowledge distillation, where a small student model learns from a pre-trained large teacher model, has achieved substantial empirical success since the seminal work of \citep{hinton2015distilling}. Despite prior theoretical studies exploring the benefits of knowledge distillation, an important question remains unanswered: why does soft-label training from the teacher require significantly fewer neurons than directly training a small neural network with hard labels? To address this, we first present motivating experimental results using simple neural network models on a binary classification problem. These results demonstrate that soft-label training consistently outperforms hard-label training in accuracy, with the performance gap becoming more pronounced as the dataset becomes increasingly difficult to classify. We then substantiate these observations with a theoretical contribution based on two-layer neural network models. Specifically, we show that soft-label training using gradient descent requires only \(O\left(\frac{1}{\gamma^2 \epsilon}\right)\) neurons to achieve a classification loss averaged over epochs smaller than some \(\epsilon > 0\), where \(\gamma\) is the separation margin of the limiting kernel. In contrast, hard-label training requires \(O\left(\frac{1}{\gamma^4} \cdot \ln\left(\frac{1}{\epsilon}\right)\right)\) neurons, as derived from an adapted version of the gradient descent analysis in \citep{ji2020polylogarithmic}. This implies that when \(\gamma \leq \epsilon\), i.e., when the dataset is challenging to classify, the neuron requirement for soft-label training can be significantly lower than that for hard-label training. Finally, we present experimental results on deep neural networks, further validating these theoretical findings.

\end{abstract}

\begin{keywords}%
  Knowledge Distillation, Projected Gradient Descent, Model Compression 
\end{keywords}

\section{Introduction}
Knowledge distillation is a popular technique for training a smaller `student' machine learning model by transferring knowledge from a large pre-trained `teacher' model.  A lightweight machine learning model is useful in many resource-constrained application scenarios, such as cyber-physical systems, mobile devices, edge computing, AR/VR, etc. due to several reasons such as limitations in memory, inference speed, and training data availability. Knowledge distillation has been proven to be a powerful solution for these challenges through model compression \citep{hinton2015distilling,compression_caruana,gou2021knowledge}. Among the various methods of knowledge distillation \citep{Gou_2021}, one prevalent approach involves using the teacher model's output logits as soft targets for training the student model. Specifically, smaller models trained with well-designed soft labels demonstrate competitive performance compared to more complex models trained with original (hard) labels. This method and its variants have been shown to be effective in many settings such as object detection \citep{NIPS2017_object_KD}, reinforcement learning \citep{NEURIPS2020_RL} and recommendation systems \citep{PAN2019137_reco}.

The empirical success of knowledge distillation has motivated many theoretical studies that aim to understand why training with soft-labels from the teacher is more effective than training with hard-labels directly. However, to the best of our knowledge, prior works (reviewed in the Related Work subsection later) do not explain why, when training a neural network, soft-label training can succeed with much fewer number of neurons than hard-label training. To the best of our knowledge, this work is the first to provide theoretical insight into this phenomenon.

Our contributions in this paper can be summarized as follows:
 \begin{enumerate}
     \item As the motivation to understand why knowledge distillation is effective, we first present (in section \ref{sec:pre experiment}) a few experimental observations using a simple model: a 2-layer fully connected neural network with ReLU activation trained for binary classification on a dataset derived from MNIST, trained via projected gradient descent. We observed that for small neural networks, soft-label training reaches higher accuracy than hard-label training. While this is a well-known phenomenon, to understand the conditions under which soft-label training is really effective, we considered two versions of the dataset, one in which there were more digits and the other where there were fewer digits. We observed that the the performance gain due to soft-label training becomes more significant when the dataset is difficult to classify (i.e., the one with more digits) {\text{red}(can we add more experiments on the two-layer neural network)}. Our theoretical models and results are motivated by these experiments and explain these experimental observations.
     \item
     We theoretically analyze the training dynamics of a 2-layer neural network student model. Assuming access to the soft labels provided by the infinite-width teacher (which can be modeled by the limiting kernel of the neural network), we train the student to minimize the cross-entropy loss to the soft labels. We show that soft-label training using projected gradient descent requires \(O\left(\frac{1}{\gamma^2 \epsilon}\right)\) neurons to achieve a classification loss averaged over epochs smaller than \(\epsilon\), where $\gamma$ is the separation margin of the limiting kernel. For the hard label setting, we adapt the gradient descent results of \citep{ji2020polylogarithmic} to our projected gradient descent framework, which requires $O(1/\gamma^4)$ neurons. This result highlights the superiority of soft-label training in terms of neuron efficiency, particularly in challenging classification scenarios when $\gamma$ is very small. 
     
     \item  Finally, we perform experiments on deep learning models with real-world datasets. Our experiments confirm that the above insight is not only valid for shallow 2-layer networks, but also deeper networks like VGG and ResNet.
 \end{enumerate}

 \subsection{Related Work}
There are many papers that aim to explain different aspects of knowledge distillation. In this section, we focus on reviewing the works that contributes to the theoretical understanding of knowledge distillation. 

A few prior works focus on linear models using the cross-entropy loss for the distillation objective. \citep{phuong2021understanding} assumes that both the student and teacher models are of the form \(f(x) = w^\top x\), where \(w\) is the weight vector of the student network, \(x\) is the input vector, and \(f(x)\) is the output. They show that, under certain conditions, the student model converges to the same weight vector as the teacher model. \citep{NEURIPS_JI_ZHU} extends this result to NTK-linearized deep networks, where the width of the student network approaches infinity. Similarly, \citep{panahi_kernel} demonstrates that the predictions from a neural network trained with soft labels from a teacher network match those from the teacher when the student is an extremely wide two-layer neural network. 

\citep{das2023understanding} study a similar setup but highlight that self-distillation can mitigate the impact of noisy data. On the other hand, \citep{mobahi_SD} explain the benefits of self-distillation through improved regularization, focusing on fitting a nonlinear function to training data within a Hilbert space. 

 However, to the best of our knowledge, none of these studies address the central question of whether soft-label training enables the use of significantly fewer neurons compared to hard-label training, which is the primary focus of our paper. A significant limitation of these linear models is that they require the student and teacher to share the same feature representations. While this assumption may be reasonable for self-distillation, it fails to capture the fundamental principle of model compression, where the student model is often much smaller and has a different feature space than the teacher.

Another line of work explores the benefits of knowledge distillation by arguing that distillation primarily acts to reduce the variance of the empirical risk relative to training with hard labels, thereby improving generalization performance. \citep{menon2020distillation} argue that using Bayes class probabilities as targets, instead of hard labels, reduces the excess risk associated with the function set learned by the student, leading to better generalization. Similarly, \citep{zhou2021rethinking_bias} propose that soft labels provide a bias-variance trade-off in the cross-entropy loss for the student. 

However, despite the simplicity of their approaches, both \citep{menon2020distillation} and \citep{zhou2021rethinking_bias} focus only on the excess risk component of the generalization error and do not provide insights into the training dynamics or behavior of the student during learning.

In summary, prior works do not address the following question: why does distillation using soft labels from a teacher lead to good performance using a smaller neural network compared to training with hard labels? Answering this question is the main focus of our paper.

\section{Preliminary Experimental Observations}
\label{sec:pre experiment}
In this section, we present preliminary experimental observations on the training of a two-layer neural network student model using both ground truth hard labels and soft labels generated by a larger teacher network. For the experiments summarized in Table \ref{tab:2 layer MNIST}, we perform binary classification on the MNIST dataset, where digits greater than 4 are labeled as class 1 and others as class 0. 

Two configurations of the MNIST data are considered: 
\begin{enumerate}
    \item \textbf{Full dataset:} Includes images of all digits (0--9).
    \item \textbf{Reduced dataset:} Excludes images corresponding to the digits \(\{1, 7, 4, 9\}\).
\end{enumerate}

The first configuration is more challenging to classify because of the difficulty in distinguishing between visually similar digits such as 1 and 7 or 4 and 9. Our observations indicate that the performance of the student model trained with soft labels remains relatively stable when transitioning from the easier dataset to the harder one. In contrast, the performance of the student model trained with hard labels shows a more significant decline. For instance, in the case of a student network with 4 neurons, the accuracy drops significantly when moving from the reduced dataset to the full dataset under hard-label training. However, with soft-label training, the performance experiences only a minimal drop in accuracy.

\begin{table}[ht]\label{tab:2 layer MNIST}
    \centering
    
    \begin{tabular}{llcccc}
        \toprule
        Dataset & Teacher Net & Student Net & Teacher Acc & St\_Hard Acc & St\_Soft Acc \\
        \midrule
        all except 1,7,4,9 & 512 & 8 & 98.75 & 97.19 & 97.67 \\
                           & 512 & 6 & 98.75 & 96.84 & 97.26 \\
                           & 512 & 4 & 98.75 & 93.04 & 95.83 \\
                           
        \midrule
        all digits         & 512 & 8 & 98.34 & 95.77 & 96.58 \\
                           & 512 & 6 & 98.34 & 95.33 & 96.00 \\
                           & 512 & 4 & 98.34 & 88.84 & 95.62 \\
                    
        \bottomrule
    \end{tabular}
    \caption{Performance Comparison on Different Datasets for the 2-layer Model. The second and third columns denote the number of neurons in the hidden layer for the teacher and student networks, respectively.}
\end{table}

\section{Theoretical Results}
In the previous section, we presented a few experimental observations on student-teacher training using a simple model of two-layer neural network. In this section, we provide theoretical insights into these observations, summarized in Theorem \ref{thm:soft label kl} and the corresponding Corollary \ref{cor: soft label requirement indicator}. We begin by introducing the student-teacher model used for training, outlining a few assumptions about the training data, and specifying the choice of distillation loss for the student to be trained using soft labels. This is followed by a description of the Projected Gradient Descent (PGD) method used for training and a theoretical analysis of the training dynamics.

\subsection{Model and Assumptions}\label{subsection:model}

We consider the following fully-connected two-layer neural network,with $m$ neurons in its hidden layer and with ReLU activation, as the student model: 
\begin{align}\label{eq:NN}
    f(x;W,a) = \frac{1}{\sqrt{m}}\sum_{j=1}^m a_j \sigma(W_j^{\top}x) =  \frac{1}{\sqrt{m}}\sum_{j=1}^m a_j \indicator(W_j^{\top}x \geq 0)W_j^{\top}x,
\end{align}
where $\sigma(z) = \max(0,z)$ is the ReLU activation function, $a_j \in \mathbb{R}$ and $W_j \in \mathbb{R}^{d}$ for $j \in [1,m]$ are respectively the final layer weight and the hidden layer weight vector corresponding to each neuron $j$ in the hidden layer. Let the vector $a$ and the matrix $W$ denote the collection of $a_j$ and $W_j$ as their $j^{\text{th}}$ element and $j^{\text{th}}$ row, respectively. Let the neural network output for an input sample $x_i$ be further denoted as $f_i(W)$ for any $i \in [n]$.

We initialized the neural network using the symmetric random initialization which was proposed in \citep{bai2020linearization} and later used in \citep{semih_et_al}. The initial parameters are given as follows: $a_j = -a_{j+\frac{m}{2}} \stackrel{\text{i.i.d.}}{\sim} \text{Unif}\{-1,1\}$ and $W(0)_j = W(0)_{j+\frac{m}{2}} \stackrel{\text{i.i.d.}}{\sim} \mathcal{N}(0,I_d)$ independent and identically distributed over $j=1,2,..,\frac{m}{2}$ and are independent from each other. For this symmetric initialization, we assume \(m\) is an even number. The symmetric initialization ensures that $f_i(W(0)) = 0, \forall i \in [n]$.

The above model of neural network is studied in detail for hard-label training in the works by \citep{ji2020polylogarithmic}, \citep{du2019gradient}, \citep{arora2019finegrained}, etc. Similar to the setting in these prior works, we fix the final layer weight vector $a$ and only train the hidden layer weight matrix $W$. 

The dataset with the ground truth values under consideration is denoted by $\mathcal{D} \coloneqq \{(x_i,y_i)\}_{i=1}^n$ for some finite integer $n$ where $x_i \in \mathbb{R}^d$ and $y_i \in  \{-1,1\}$. For simplicity, assume $\|x_i\|_2 \leq 1, \forall i \in [1,n]$, which is standard in prior works. Let $x_i ,\forall i\in [n]$ also satisfy: $\norm{x_i} \geq c$ for some $c > 0$.

We make the following assumption  on the dataset $\mathcal{D}$ characterizing the separability by the corresponding infinite-width NTK as in \cite{ji2020polylogarithmic}:
Let $\mu_{\mathcal{N}}$ be the Gaussian measure on $\mathcal{R}^d$, given by the Gaussian density with respect to the Lebesgue measure on $\mathbb{R}^d$. We consider the following Hilbert space
\begin{equation}
    \mathcal{H} \coloneqq \left\{w : \mathbb{R}^d \rightarrow \mathbb{R}^d \Bigg| \int \norm{w(z)}_2^2 \,d\mu_{\mathcal{N}}(z) < \infty \right\}.
\end{equation}
For any $x \in \mathbb{R}^d$, define $\phi_x \in \mathcal{H}$ by $\phi_x(z) \coloneqq \indicator\left(\langle z,x \rangle > 0\right)x$.
\begin{assumption}\label{ass:RKHS}
    There exists $\overline{v} \in \mathcal{H}$ and $\gamma > 0$, such that $\norm{\overline{v}(z)}_2 \leq 1$ for any $z \in \mathcal{R}^d$, and for any $1 \leq i \leq n$,
    \begin{equation}
        y_i \langle \overline{v}, \phi_i \rangle_{\mathcal{H}} \coloneqq y_i \int \langle \overline{v}(z), \phi_i(z) \rangle \,d\mu_{\mathcal{N}}(z) \geq \gamma.
    \end{equation}
\end{assumption}
Using standard notations, $\phi_x(W_j(0)) = \indicator(W_j(0)^{\top}x >0)x$, for all $j \in [m]$ are called the NTK-features for input data $x$. Let us denote $\overline{\phi}_x^0\in \mathbb{R}^{md \times 1}$ as a concatenation of all $\phi_x(W_j(0))$ for input data $x$. The above assumption ensures the separability of the induced set $\left\{\overline{\phi}_{x_i}^0,y_i\right\}_{i=1}^n$ when $m$ is sufficiently large (see Lemma 2.3 in \citep{ji2020polylogarithmic}). As shown in \citep{ji2020polylogarithmic}, there is always a $\gamma >0$ satisfying assumption \ref{ass:RKHS} as long as no two inputs $x_i$ and $x_j$ with $i,j\in [1,n], i\neq j$ are parallel in the dataset $\mathcal{D}$.

For our theoretical results, we define the soft labels for the student model based on the Hilbert space \(\mathcal{H}\), as described below.
 
 Let \(v : \mathbb{R}^d \rightarrow \mathbb{R}^d\) be a function in the RKHS \(\mathcal{H}\) such that \(\|v(z)\|_2 \leq 1\) for all \(z \in \mathbb{R}^d\), and 
\(
y_i \mathbb{E}_{z \sim \mathcal{N}(0, I_d)} (\phi_i(z)^\top v(z)) \geq \gamma.
\)
Assumption~\ref{ass:RKHS} ensures the existence of such a \(v(z)\). Define for each $i \in [n]$ 
\[
z_i \coloneqq \mathbb{E}_{z \sim \mathcal{N}(0, I_d)} (\phi_i(z)^\top v(z)).
\]
The soft labels for each input \(x_i\) are then given by 
\[
p_i = \mu(z_i) \coloneqq \frac{1}{1 + \exp(-z_i)},
\]
where \(\mu(z)\) denotes the softmax function applied to the target logit \(z_i\). In other words, we assume that we have access to the soft labels from a teacher which is the kernel limit of an inifnite-width neural network. We fix the function \(v\) for the remainder of the paper. The choice of a particular \(v\) does not affect our analysis, provided it satisfies the conditions described above.

Our training objective is to minimize the following empirical risk over the dataset $\mathcal{D}$:
\begin{equation}\label{eq:soft loss}
 R^{KL}(W) \coloneqq \frac{1}{n}\sum_{i=1}^n\ell^{KL}(p_i,f_i(W)),
\end{equation}
where $\ell^{KL}(p_i,f_i(W))$ is the Kullback-Leibler (KL) divergence between the soft-label $p_i$ and the softmax version of the network output $f_i(W)$: $$\ell^{KL}(p_i,f_i(W)) \coloneqq p_i \ln \left(\mu(f_i(W)) /p_i\right) + (1-p_i) \ln \left((1-\mu(f_i(W)))/(1-p_i)\right).$$ 

 We now discuss the student training procedure. The student is trained on soft labels using a projected gradient descent (PGD) algorithm. First, we define $\nabla_{W_j} R^{KL}(W)$ which serves as the gradient of the risk with respect to the weight for the $j^{\text{th}}$ neuron:
\begin{align}
    \nabla_{W_j} R^{KL}(W) = \frac{1}{n}\sum_{i=1}^n \nabla_{f}\ell^{KL}(p_i,f_i(W))\nabla_{W_j}f_i(W)
\end{align} 
where, 
\begin{equation*}
    \nabla_{W_j}f_i(W)  = \frac{a_j}{\sqrt{m}}\phi_i(W_j) = \frac{a_j}{\sqrt{m}}\indicator(W_j^{\top}x_i > 0).
\end{equation*} 
Note that the gradient of the loss is defined for the whole domain of $W_j$ in $\mathcal{R}^d$ even though the ReLU activation function is not differentiable at 0. Let $W(t)$ denotes the weight matrix of the neural network after $t^{\text{th}}$ iteration of training. Let the feasible set of weights be defined as 
\(
S_B \coloneqq \{W: \|W_j - W_j(0)\|_2 \leq \frac{B}{\sqrt{m}}\},
\)
where \( B \) is a hyperparameter in the training. The PGD algorithm updates the weights per iteration using the following steps:

\begin{enumerate}
    \item \textbf{Descent step:} 
    \(
    \hat{W}_j(t+1) = W_j(t) - \eta \nabla_{W_j} R^{KL}(W_j(t)), \quad \forall t \geq 0.
    \)

    \item \textbf{Projection step:} 
    \(
    W_j(t+1) = \text{Projection of } \hat{W}_j(t+1) \text{ into } S_B.
    \)
\end{enumerate}

In the next subsection we provide the key theoretical result of this paper. In Theorem~\ref{thm:soft label kl}, we characterize the neuron requirement to achieve an arbitrarily small training loss averaged over epochs. We then compare this with the neuron requirement for hard-label training, as established for Gradient Descent in \citep{ji2020polylogarithmic}, by adapting their result to Projected Gradient Descent in this paper.

\subsection{Neuron Requirement for Soft-label Training}
We are now ready to present the first result of this paper. The following Theorem provides a characterization of the number of neurons required for soft-label training to achieve a small empirical risk, \(R^{KL}(W(t))\), averaged over iterations \(t < T\).

\begin{theorem}\label{thm:soft label kl}
Let \(\beta \in (0,1)\), \(\delta \in \left(0,\frac{1}{3}\right)\) be fixed real numbers. If the number of neurons \(m\) satisfies
\begin{equation}\label{eq:m requirement soft kl}
   m \geq \frac{C_1}{\beta}\left(\sqrt{\frac{2}{\pi}}\frac{1}{c} + 3\sqrt{\ln\left(\frac{2n}{\delta}\right)} \right)^2,
\end{equation} 
and the PGD algorithm is run with a projection radius \(B = 1\) for \(T\) iterations such that 
\(
T \geq \frac{9}{\beta^2},
\) 
using a constant step size \(\eta\) satisfying \(\eta \leq \frac{\beta}{3}\). Then, the following bound on the averaged empirical risk holds:
\begin{equation}
    \frac{1}{T}\sum_{\tau < T}R^{KL}(W(\tau)) \leq \beta,
\end{equation}
with probability at least \(1 - 3\delta\) over the random initialization. Here, \(C_1 = \frac{96}{(1+e^2)}\) is an absolute constant.
\end{theorem}

In light of the above Theorem, we next analyze the performance of the student network in classifying the training data. We define the classification error as:
\begin{equation}\label{eq:indicator loss}
    R(W) \coloneqq R(W; \mathcal{D}) \coloneqq \frac{1}{n}\sum_{i=1}^n \indicator(y_i f_i(W) > 0).
\end{equation}

Using a reverse Pinsker inequality (Lemma 4.1 of \citep{G_tze_2019}), we provide an upper bound on \( R(W) \) in terms of \( R^{KL}(W) \):

\begin{lemma}\label{lemma:relating indicator to kl}
    If \( 0 < p_i < 1 \), where \( p_i \) is the soft label in the above construction for the data input \( x_i \) for \( i \in [n] \), then the classification loss \( R(W) \) can be related to the surrogate loss \( R^{KL}(W) \) as follows:
    \begin{equation}
        R(W(t)) \leq \frac{32}{\gamma^2}R^{KL}(W(t)).
    \end{equation}
\end{lemma}

Combining the above Lemma with Theorem~\ref{thm:soft label kl}, we arrive at the following Corollary:

\begin{corollary}\label{cor: soft label requirement indicator}
    Let \( \epsilon \in (0,1) \) and \( \delta \in \left(0,\frac{1}{3}\right) \) be fixed real numbers. If the number of neurons satisfies:
    \begin{equation}\label{eq:m requirement soft indicator}
        m \geq \frac{C_2}{\gamma^2 \epsilon} \left(\sqrt{\frac{2}{\pi}}\frac{1}{c} + 3\sqrt{\ln\left(\frac{2n}{\delta}\right)}\right)^2,
    \end{equation}
    then the PGD algorithm with \( B = 1 \), step size \( \eta \leq \frac{\gamma^2 \epsilon}{3} \), and iteration count \( T \geq \frac{9}{\gamma^4 \epsilon^2} \) ensures the following guarantee on the classification loss with probability at least \( 1 - 3\delta \):
    \begin{equation}
        \frac{1}{T} \sum_{\tau < T} R(W(\tau)) \leq \epsilon,
    \end{equation}
    where \( C_2 = 32\cdot C_1 \).
\end{corollary}

\subsubsection{Comparison with Hard-label Training}
Now we are ready to compare the neuron requirement based on Corollary \ref{cor: soft label requirement indicator} with that of the requirement for hard label training as established in \citep{ji2020polylogarithmic}. The empirical risk for hard-label training on the dataset \(\mathcal{D} \coloneqq \{x_i, y_i\}_{i=1}^n\) is defined as:
\begin{equation}\label{eq:hard loss}
     R^{h}(W) \coloneqq \frac{1}{n}\sum_{i=1}^n \ln(1+\exp(-y_if_i(W))).
\end{equation}

The following proposition for hard-label training is adapted from the result on gradient descent with hard labels by \citep{ji2020polylogarithmic}, modified to fit our projected gradient descent (PGD) setting. The modification introduces a projection step after each weight update, where the weight of each neuron is constrained to the set \(W_j - W_j(0) \leq \frac{B}{\sqrt{m}}\), for some hyperparameter \(B > 0\). The proof of Proposition \ref{prop: hard label PGD} is provided in the Appendix.

\begin{proposition}\label{prop: hard label PGD}
    Fix \(\beta \in (0,1)\) and \(\delta \in \left(0,\frac{1}{3}\right)\). With a choice of the projection ball radius \(B = \frac{2}{\gamma}\ln \left(\frac{2}{\ln(2)\beta}\right)\), if the number of training iterations satisfies 
    \(
    T \geq \frac{8}{\gamma^2 \eta}\frac{\ln^2 \left(\frac{2}{\ln(2)\beta}\right)}{\ln(2)\beta}
    \), $\eta \leq 1$,
    and the number of neurons satisfies:
    \begin{equation}\label{eq:m requirement hard}
        m \geq \frac{16}{\gamma^4}\left(\frac{2\sqrt{2}}{c\sqrt{\pi}}\ln\left(\frac{2}{\ln(2)\beta}\right) + 3\sqrt{\ln\left(\frac{2n}{\delta}\right)}\right)^2,
    \end{equation}
    then the following holds with probability at least \(1-3\delta\)
    \begin{equation}\label{eq:hard label surrogate small}
        \frac{1}{T}\sum_{\tau < T}R^h(W(\tau)) \leq  \beta.
    \end{equation}
\end{proposition}

The hard-label surrogate loss guarantee in equation~\eqref{eq:hard label surrogate small} implies the following classification loss guarantee:
\[
\frac{1}{T}\sum_{\tau < T}R(W(\tau)) \leq \frac{1}{\ln(2)}\beta.
\]

It is worth mentioning that, without the projection step, the neuron requirement for hard-label training, as established by \citep{ji2020polylogarithmic}, is \(O\left(\frac{1}{\gamma^8}, \ln\left(\frac{n}{\delta}\right), \ln\left(\frac{1}{\epsilon}\right)\right)\).

\textbf{Comparing the neuron requirements:} Based on Corollary~\ref{cor: soft label requirement indicator} and Proposition~\ref{prop: hard label PGD}, the requirements for hard-label and soft-label training to ensure the classification loss averaged over iterations is less than some \(\epsilon > 0\) can be expressed as:
\[
O\left(\frac{1}{\gamma^4}, \ln\left(\frac{n}{\delta}\right), \ln\left(\frac{1}{\epsilon}\right)\right) \quad \text{(hard labels)},
\]
and
\[
O\left(\frac{1}{\gamma^2 \epsilon} \ln\left(\frac{n}{\delta}\right), \frac{1}{\beta}\right) \quad \text{(soft labels)}.
\]

Suppose that there is a target classification error of epsilon. The above results suggest that, when \(\gamma < \epsilon\), i.e., when the data-set is more difficult the separate, soft-label training reduces the neuron requirement by a factor of \(\gamma\) in the regime of PGD training. In other words, when the separation margin \(\gamma\) associated with the classification task is sufficiently small, soft-label training requires significantly fewer neurons to achieve similar performance compared to hard-label training.

\subsection{Proof Sketch of Theorem \ref{thm:soft label kl}}
While the detailed proofs of all the results presented in the paper is provided in the Appendix, we present a high-level proof sketch for Theorems \ref{thm:soft label kl} here. Before presenting the proof sketch, we first introduce some additional notations and quantities. 

 Similar to the definition of \(\overline{\phi}^0_i\), define the feature map at iteration \(t\), \(\overline{\phi}^t_i\), based on the weight \(W(t)\). Specifically, the feature corresponding to the \(j^{\text{th}}\) neuron is given by \(\phi_i(W_j(t))\). Now, define \(f_i^t(W)\) for each data sample \(x_i\) as:
\[
f_i^t(W) \coloneqq \frac{1}{\sqrt{m}} \sum_{j=1}^m a_j \indicator(W_j(t)^\top x_i \geq 0) W_j^\top x_i = \frac{1}{\sqrt{m}} \langle \overline{\phi}^t_i, W \rangle,
\]
where \(W_j(t)\) is the weight of the \(j^{\text{th}}\) neuron at the \(t^{\text{th}}\) iteration of PGD.

Using these definitions, we define the following expression for each \(t \leq T\):
\[
R^{t,KL}(W) \coloneqq \frac{1}{n} \sum_{i=1}^n \ell^{KL}(p_i, f_i^t(W)).
\]

 \textbf{Initial Feature Map and Separability:}
 The following Lemma \ref{lemma:sub sample} implies the existence of a weight matrix \(U\) satisfying \(\|U_j - W_j(0)\|_2 \leq \frac{1}{\sqrt{m}}\) such that \(|\langle \overline{\phi}^0_i, U \rangle - z_i|\) is small for all \(i \in [n]\), with high probability, when \(m\) is sufficiently large.
    
    \begin{lemma}\label{lemma:sub sample}
Let \(U \in \mathbb{R}^{m \times d}\) be defined as \(U_j = \frac{a_j}{\sqrt{m}}v(W_j(0)), \forall j \in [m]\). Then, under Assumption~\ref{ass:RKHS}, for the symmetric random initialization of the neural network weights, the following holds with probability at least \(1-\delta\):
\begin{equation}\label{eq:sub_sample_1}
   \left| f_i^0(U) - z_i \right| \leq \frac{1}{\sqrt{m}} \sqrt{2\ln(2n/\delta)}, \quad \forall i \in [1,n],
\end{equation}

\end{lemma}

This observation suggests that a linear function (linear in the weight \(W\)) of the form \(\langle \overline{\phi}^0_i, W \rangle\) can approximate the soft label \(z_i\) for each \(i\) with high probability. This intuition is crucial for the subsequent steps of the proof.

 \textbf{Convergence of Soft Label Surrogate Loss:}
   Next we show that under sufficient conditions on $T$ and $\eta$, the soft-label risk averaged over all iterations converges to the quantity \(\frac{1}{T} \sum_{t < T} R^{t,KL}(\overline{W})\) for any \(\overline{W}\) in the feasible set \(S_B\). 

The remainder of the proof is devoted to showing that \(R^{t,KL}(\overline{W})\) is small for all \(t \leq T\) with high probability, for an appropriately chosen \(\overline{W}\).

\textbf{Bounding $R^{t,KL}(\overline{W})$:}
   The idea is to bound $\ell^{KL}(p_i, f_i^t(\overline{W}))$ for each $i$ with high probability. For this purpose, we use a reverse Pinsker's inequality, as stated in Lemma 4.1 of \citep{G_tze_2019}. This inequality allows us to upper-bound the KL divergence loss \(\ell^{KL}(p_i, f_i^t(\overline{W}))\) for each \(i\) by the distance between the corresponding logits, \(|z_i - f_i^t(\overline{W})|\).

Now, choosing \(\overline{W} = W(0) + U\) and applying the triangle inequality, we obtain:
\[
|z_i - f_i^t(\overline{W})| \leq |z_i - f_i^0(U)| + |f_i^t(U) - f_i^0(U)| + |f_i^t(W(0)) - f_i^0(W(0))| .
\]

From Lemma~\ref{lemma:sub sample}, we know that \(\|z_i - f_i^0(U)\|_2\) is small under the sufficient conditions. Therefore, it remains to show that the terms \(|f_i^t(U) - f_i^0(U)|\) and $|f_i^t(W(0)) - f_i^0(W(0))|$ are both small for each \(i\) with high probability. We use the analysis in the neural tangent kernel literature (\citep{jacot2020neural,chen2020generalized}) for bounding this expression.

\section{More Experimental Results}

In this section, we validate our results using deep neural networks, with VGG 8+3 as the teacher and VGG 2+3 as the student; see \citep{simonyan2014very} for a detailed description of the VGG architecture. The dataset chosen for our experiments is the CIFAR-10 cat/dog dataset. 

Since our theory suggests that harder-to-classify datasets benefit more from soft-label training, we added Gaussian noise to the CIFAR-10 cat/dog dataset to make it more challenging to classify. We compare the performance of soft-label vs. hard-label training across different datasets: the original CIFAR-10 dataset and noise-added CIFAR-10 datasets. The results are summarized in Figure \ref{fig:vgg test}. Consistent with our theoretical predictions, the experiments demonstrate that harder-to-classify datasets benefit significantly more from distillation. 

All experimental points are averaged over 10 independent runs. For each run, we employed early stopping with a patience of 20 epochs and a maximum of 100 iterations. The dataset was split into training, validation, and test sets with proportions of 80\%, 10\%, and 10\%, respectively. We trained the models using gradient descent with the Adam optimizer and applied \(L_2\) regularization.

\begin{figure}[ht]\label{fig:vgg test}
    \centering
    \includegraphics[width=0.65\textwidth]{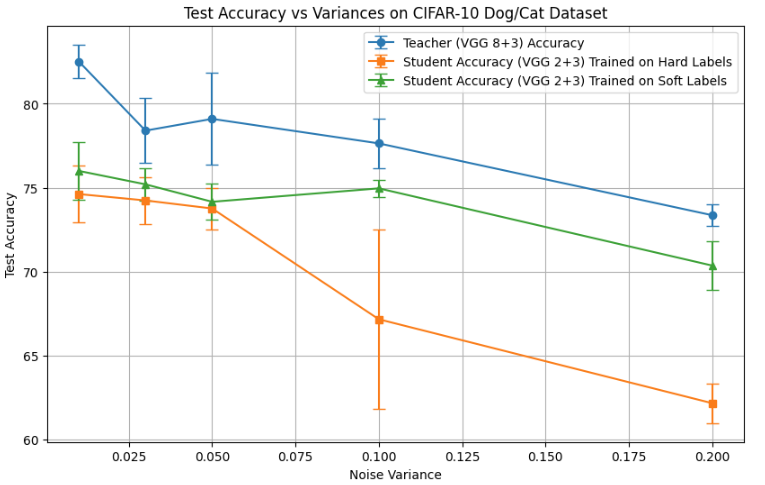} 
    \caption{Classification accuracy under Gaussian noise on CIFAR-10 cat/dog with VGG 8+3 as the teacher and VGG 2+3 as the student.}
    \label{fig:gaussian_noise_cifar10}
\end{figure}

\section{Conclusions}
In this paper, we provide theoretical results which show that soft-label training leads to fewer number of neurons than hard--label training for the same training accuracy. Our proofs provide some intuition for this phenomenon, as stated at the end of the longer version of the paper, which we summarize here. The parameters of a neural network have a dual role: one is to identify good features and the other is to assign weights to these features. With good initialization, one can start the training with good features. In contrast to hard-label training, soft-label training ensures that the network parameters do not deviate too much from initialization, thus approximately maintaining good features throughout the training process, but they deviate just enough to assign the right weights to various features. 

\bibliography{iclr2025_conference}

\begin{thebibliography}{24}
\providecommand{\natexlab}[1]{#1}
\providecommand{\url}[1]{\texttt{#1}}
\expandafter\ifx\csname urlstyle\endcsname\relax
  \providecommand{\doi}[1]{doi: #1}\else
  \providecommand{\doi}{doi: \begingroup \urlstyle{rm}\Url}\fi

\bibitem[Allen-Zhu and Li(2019)]{allen2019can}
Zeyuan Allen-Zhu and Yuanzhi Li.
\newblock What can resnet learn efficiently, going beyond kernels?
\newblock \emph{Advances in Neural Information Processing Systems}, 32, 2019.

\bibitem[Arora et~al.(2019)Arora, Du, Hu, Li, and Wang]{arora2019finegrained}
Sanjeev Arora, Simon~S. Du, Wei Hu, Zhiyuan Li, and Ruosong Wang.
\newblock Fine-grained analysis of optimization and generalization for overparameterized two-layer neural networks, 2019.

\bibitem[Bai and Lee(2020)]{bai2020linearization}
Yu~Bai and Jason~D. Lee.
\newblock Beyond linearization: On quadratic and higher-order approximation of wide neural networks, 2020.

\bibitem[Buciluundefined et~al.(2006)Buciluundefined, Caruana, and Niculescu-Mizil]{compression_caruana}
Cristian Buciluundefined, Rich Caruana, and Alexandru Niculescu-Mizil.
\newblock Model compression.
\newblock In \emph{Proceedings of the 12th ACM SIGKDD International Conference on Knowledge Discovery and Data Mining}, KDD '06, page 535–541, New York, NY, USA, 2006. Association for Computing Machinery.
\newblock ISBN 1595933395.
\newblock \doi{10.1145/1150402.1150464}.
\newblock URL \url{https://doi.org/10.1145/1150402.1150464}.

\bibitem[Cayci et~al.(2023)Cayci, Satpathi, He, and Srikant]{semih_et_al}
Semih Cayci, Siddhartha Satpathi, Niao He, and R.~Srikant.
\newblock Sample complexity and overparameterization bounds for temporal-difference learning with neural network approximation.
\newblock \emph{IEEE Transactions on Automatic Control}, 68\penalty0 (5):\penalty0 2891--2905, 2023.
\newblock \doi{10.1109/TAC.2023.3234234}.

\bibitem[Chen et~al.(2017)Chen, Choi, Yu, Han, and Chandraker]{NIPS2017_object_KD}
Guobin Chen, Wongun Choi, Xiang Yu, Tony Han, and Manmohan Chandraker.
\newblock Learning efficient object detection models with knowledge distillation.
\newblock In I.~Guyon, U.~Von Luxburg, S.~Bengio, H.~Wallach, R.~Fergus, S.~Vishwanathan, and R.~Garnett, editors, \emph{Advances in Neural Information Processing Systems}, volume~30. Curran Associates, Inc., 2017.
\newblock URL \url{https://proceedings.neurips.cc/paper_files/paper/2017/file/e1e32e235eee1f970470a3a6658dfdd5-Paper.pdf}.

\bibitem[Chen et~al.(2020)Chen, Cao, Gu, and Zhang]{chen2020generalized}
Zixiang Chen, Yuan Cao, Quanquan Gu, and Tong Zhang.
\newblock A generalized neural tangent kernel analysis for two-layer neural networks.
\newblock \emph{Advances in Neural Information Processing Systems}, 33:\penalty0 13363--13373, 2020.

\bibitem[Das and Sanghavi(2023)]{das2023understanding}
Rudrajit Das and Sujay Sanghavi.
\newblock Understanding self-distillation in the presence of label noise, 2023.

\bibitem[Du et~al.(2019)Du, Zhai, Poczos, and Singh]{du2019gradient}
Simon~S. Du, Xiyu Zhai, Barnabas Poczos, and Aarti Singh.
\newblock Gradient descent provably optimizes over-parameterized neural networks, 2019.

\bibitem[Gou et~al.(2021{\natexlab{a}})Gou, Yu, Maybank, and Tao]{Gou_2021}
Jianping Gou, Baosheng Yu, Stephen~J. Maybank, and Dacheng Tao.
\newblock Knowledge distillation: A survey.
\newblock \emph{International Journal of Computer Vision}, 129\penalty0 (6):\penalty0 1789–1819, March 2021{\natexlab{a}}.
\newblock ISSN 1573-1405.
\newblock \doi{10.1007/s11263-021-01453-z}.
\newblock URL \url{http://dx.doi.org/10.1007/s11263-021-01453-z}.

\bibitem[Gou et~al.(2021{\natexlab{b}})Gou, Yu, Maybank, and Tao]{gou2021knowledge}
Jianping Gou, Baosheng Yu, Stephen~J Maybank, and Dacheng Tao.
\newblock Knowledge distillation: A survey.
\newblock \emph{International Journal of Computer Vision}, 129\penalty0 (6):\penalty0 1789--1819, 2021{\natexlab{b}}.

\bibitem[Götze et~al.(2019)Götze, Sambale, and Sinulis]{G_tze_2019}
Friedrich Götze, Holger Sambale, and Arthur Sinulis.
\newblock Higher order concentration for functions of weakly dependent random variables.
\newblock \emph{Electronic Journal of Probability}, 24\penalty0 (none), January 2019.
\newblock ISSN 1083-6489.
\newblock \doi{10.1214/19-ejp338}.
\newblock URL \url{http://dx.doi.org/10.1214/19-EJP338}.

\bibitem[Hinton et~al.(2015)Hinton, Vinyals, and Dean]{hinton2015distilling}
Geoffrey Hinton, Oriol Vinyals, and Jeff Dean.
\newblock Distilling the knowledge in a neural network, 2015.

\bibitem[Jacot et~al.(2020)Jacot, Gabriel, and Hongler]{jacot2020neural}
Arthur Jacot, Franck Gabriel, and Clément Hongler.
\newblock Neural tangent kernel: Convergence and generalization in neural networks, 2020.

\bibitem[Ji and Zhu(2020)]{NEURIPS_JI_ZHU}
Guangda Ji and Zhanxing Zhu.
\newblock Knowledge distillation in wide neural networks: Risk bound, data efficiency and imperfect teacher.
\newblock In H.~Larochelle, M.~Ranzato, R.~Hadsell, M.F. Balcan, and H.~Lin, editors, \emph{Advances in Neural Information Processing Systems}, volume~33, pages 20823--20833. Curran Associates, Inc., 2020.
\newblock URL \url{https://proceedings.neurips.cc/paper_files/paper/2020/file/ef0d3930a7b6c95bd2b32ed45989c61f-Paper.pdf}.

\bibitem[Ji and Telgarsky(2020)]{ji2020polylogarithmic}
Ziwei Ji and Matus Telgarsky.
\newblock Polylogarithmic width suffices for gradient descent to achieve arbitrarily small test error with shallow relu networks, 2020.

\bibitem[Menon et~al.(2020)Menon, Rawat, Reddi, Kim, and Kumar]{menon2020distillation}
Aditya~Krishna Menon, Ankit~Singh Rawat, Sashank~J. Reddi, Seungyeon Kim, and Sanjiv Kumar.
\newblock Why distillation helps: a statistical perspective, 2020.

\bibitem[Mobahi et~al.(2020)Mobahi, Farajtabar, and Bartlett]{mobahi_SD}
Hossein Mobahi, Mehrdad Farajtabar, and Peter~L. Bartlett.
\newblock Self-distillation amplifies regularization in hilbert space.
\newblock In \emph{Proceedings of the 34th International Conference on Neural Information Processing Systems}, NIPS '20, Red Hook, NY, USA, 2020. Curran Associates Inc.
\newblock ISBN 9781713829546.

\bibitem[Pan et~al.(2019)Pan, He, and Yu]{PAN2019137_reco}
Yiteng Pan, Fazhi He, and Haiping Yu.
\newblock A novel enhanced collaborative autoencoder with knowledge distillation for top-n recommender systems.
\newblock \emph{Neurocomputing}, 332:\penalty0 137--148, 2019.
\newblock ISSN 0925-2312.
\newblock \doi{https://doi.org/10.1016/j.neucom.2018.12.025}.
\newblock URL \url{https://www.sciencedirect.com/science/article/pii/S0925231218314796}.

\bibitem[Panahi et~al.(2022)Panahi, Rahbar, Bhattacharyya, Dubhashi, and Haghir~Chehreghani]{panahi_kernel}
Ashkan Panahi, Arman Rahbar, Chiranjib Bhattacharyya, Devdatt Dubhashi, and Morteza Haghir~Chehreghani.
\newblock Analysis of knowledge transfer in kernel regime.
\newblock In \emph{Proceedings of the 31st ACM International Conference on Information \& Knowledge Management}, CIKM '22, page 1615–1624, New York, NY, USA, 2022. Association for Computing Machinery.
\newblock ISBN 9781450392365.
\newblock \doi{10.1145/3511808.3557237}.
\newblock URL \url{https://doi.org/10.1145/3511808.3557237}.

\bibitem[Phuong and Lampert(2021)]{phuong2021understanding}
Mary Phuong and Christoph~H. Lampert.
\newblock Towards understanding knowledge distillation, 2021.

\bibitem[Simonyan and Zisserman(2014)]{simonyan2014very}
Karen Simonyan and Andrew Zisserman.
\newblock Very deep convolutional networks for large-scale image recognition.
\newblock \emph{arXiv preprint arXiv:1409.1556}, 2014.

\bibitem[Xu et~al.(2020)Xu, Wu, Che, Tang, and Ye]{NEURIPS2020_RL}
Zhiyuan Xu, Kun Wu, Zhengping Che, Jian Tang, and Jieping Ye.
\newblock Knowledge transfer in multi-task deep reinforcement learning for continuous control.
\newblock In H.~Larochelle, M.~Ranzato, R.~Hadsell, M.F. Balcan, and H.~Lin, editors, \emph{Advances in Neural Information Processing Systems}, volume~33, pages 15146--15155. Curran Associates, Inc., 2020.
\newblock URL \url{https://proceedings.neurips.cc/paper_files/paper/2020/file/acab0116c354964a558e65bdd07ff047-Paper.pdf}.

\bibitem[Zhou et~al.(2021)Zhou, Song, Chen, Zhou, Wang, Yuan, and Zhang]{zhou2021rethinking_bias}
Helong Zhou, Liangchen Song, Jiajie Chen, Ye~Zhou, Guoli Wang, Junsong Yuan, and Qian Zhang.
\newblock Rethinking soft labels for knowledge distillation: A bias-variance tradeoff perspective, 2021.

\end{thebibliography}

\appendix
\section{Proof of Theorem \ref{thm:soft label kl}}
In this section, we provide the detailed proof of Theorem \ref{thm:soft label kl}. A proof sketch is given as follows:
\begin{enumerate}
    \item \textbf{Initial Feature Map and Separability:}
    Recall that the feature map given by the initial weights for a data input \(x_i\) is denoted by \(\overline{\phi}^0_i\). Lemma \ref{lemma:sub sample} shows that this feature map, based on the initial weights, separates the dataset with a margin of order \(O(\gamma)\), provided the number of neurons is of the order \(O\left(\frac{1}{\gamma^2}\right)\). Lemma \ref{lemma:sub sample} further implies the existence of a weight matrix \(U\) satisfying \(\|U_j - W_j(0)\|_2 \leq \frac{1}{\sqrt{m}}\) such that \(|\langle \overline{\phi}^0_i, U \rangle - z_i|\) is small for all \(i \in [n]\), with high probability, when \(m\) is sufficiently large. Notice that Lemma \ref{lemma:sub sample} is a extended version of Lemma \ref{lemma:sub sample} stated in the main part of the paper.

This observation suggests that a linear function (linear in the weight \(W\)) of the form \(\langle \overline{\phi}^0_i, W \rangle\) can approximate the soft label \(z_i\) for each \(i\) with high probability. This intuition is crucial for the subsequent steps of the proof.

    \item \textbf{Convergence of Soft Label Surrogate Loss:}
    Next, we provide a convergence result for the soft-label surrogate loss in Lemma~\ref{lemma:descent soft} under PGD updates over iterations. To describe the statement in Lemma~\ref{lemma:descent soft}, we first introduce some additional notations and quantities. 

Define \(f_i^t(W)\) for each data sample \(x_i\) as:
\[
f_i^t(W) \coloneqq \frac{1}{\sqrt{m}} \sum_{j=1}^m a_j \indicator(W_j(t)^\top x_i \geq 0) W_j^\top x_i,
\]
where \(W_j(t)\) is the weight of the \(j^{\text{th}}\) neuron at the \(t^{\text{th}}\) iteration of PGD. Similar to the definition of \(\overline{\phi}^0_i\), define the feature map at iteration \(t\), \(\overline{\phi}^t_i\), based on the weight \(W(t)\). Specifically, the feature corresponding to the \(j^{\text{th}}\) neuron is given by \(\phi_i(W_j(t))\). Hence, for all \(t \leq T\) and \(i \in [n]\), we can write:
\[
f_i^t(W) = \frac{1}{\sqrt{m}} \langle \overline{\phi}^t_i, W \rangle.
\]
Using these definitions, we define the following expression for each \(t \leq T\):
\[
R^{t,KL}(W) \coloneqq \frac{1}{n} \sum_{i=1}^n \ell^{KL}(p_i, f_i^t(W)).
\]

We are now ready to state Lemma~\ref{lemma:descent soft}. This lemma shows that if the number of iterations \(T\) is sufficiently large and the step size \(\eta\) is small, the soft-label risk averaged over all iterations converges to the quantity \(\frac{1}{T} \sum_{t < T} R^{t,KL}(\overline{W})\) for any \(\overline{W}\) in the feasible set \(S_B\). 

The remainder of the proof is devoted to showing that \(R^{t,KL}(\overline{W})\) is small for all \(t \leq T\) with high probability, for an appropriately chosen \(\overline{W}\), provided the number of neurons \(m\) is sufficiently large. This ensures that the soft-label risk averaged over all iterations converges to a small value if PGD is run for a sufficient number of iterations under suitable conditions on \(m\).

   \item \textbf{Bounding $R^{t,KL}(\overline{W})$:}
   The approach to bounding \(R^{t,KL}(\overline{W})\) for each \(i \in [1,n]\) relies on using a reverse Pinsker's inequality, as stated in Lemma 4.1 of \citep{G_tze_2019}. This inequality allows us to upper-bound the KL divergence loss \(\ell^{KL}(p_i, f_i^t(\overline{W}))\) for each \(i\) by the distance between the corresponding logits, \(|z_i - f_i^t(\overline{W})|\). Since \(f_i^t(\overline{W}) = \frac{1}{\sqrt{m}} \langle \overline{\phi}^t_i, \overline{W} \rangle\), we can write at \(t = 0\):
\[
f_i^0(\overline{W}) = \frac{1}{\sqrt{m}} \langle \overline{\phi}^0_i, \overline{W} \rangle.
\]

Now, choosing \(\overline{W} = W(0) + U\) and applying the triangle inequality, we obtain:
\[
|z_i - f_i^t(\overline{W})| \leq |z_i - f_i^0(U)| + |f_i^t(U) - f_i^0(U)| + |f_i^t(W(0)) - f_i^0(W(0))| .
\]

From Lemma~\ref{lemma:sub sample}, we know that \(\|z_i - f_i^0(U)\|_2\) is small under the sufficient conditions. Therefore, it remains to show that the terms \(|f_i^t(U) - f_i^0(U)|\) and $|f_i^t(W(0)) - f_i^0(W(0))|$ are both small for each \(i\) with high probability.

\item \textbf{Bounding $|f_i^t(U) - f_i^0(U)|$ and $|f_i^t(W(0)) - f_i^0(W(0))|$:}
For the final part of the proof, we establish in Lemma~\ref{lemma:NTK approx error} that the terms \(|f_i^t(U) - f_i^0(U)|\) and $|f_i^t(W(0)) - f_i^0(W(0))|$ are both small with high probability under our PGD setting. Combining this result with previous steps, we conclude that as long as \(m\) is of the order \(O\left(\frac{1}{\epsilon}\right)\), and under mild sufficient conditions on \(T\) (the number of iterations) and \(\eta\) (the step size), the soft-label loss averaged over all iterations is upper-bounded by \(\epsilon\), for any \(\epsilon > 0\).

\end{enumerate}
\subsection{Proof of Lemma \ref{lemma:sub sample}}
 Under Assumption~\ref{ass:RKHS}, the random initialization of the neural network acts as a feature map that separates the dataset with a margin of order \(O\left(\gamma\right)\), provided the number of neurons is of the order \(O\left(\frac{1}{\gamma^2}\right)\). Separability of initial features is established in Lemma 2.3 of \citep{ji2020polylogarithmic} for random initialization of the parameters. The proof for the symmetric random initialization is a slight modification of the proof for random initialization is presented after the lemma for completeness.

     We can write:
\begin{align*}
    f^0_i(U) = 2\frac{1}{\sqrt{m}}\sum_{j=1}^{m/2} a_j \phi_i(W_j(0))^{\top} \frac{a_j}{\sqrt{m}}v(W_j(0)) = \frac{2}{m}\sum_{j=1}^{m/2}\phi_i(W_j(0))^{\top}v(W_j(0))
\end{align*}
Now as $W_j(0)$ are independently drawn from $\mathcal{N}(0,I_d)$ for $j \in [1,m/2]$, using Hoeffding inequality, we can write:
Over the randomly symmetric initialization , under assumption \ref{ass:RKHS}, for each $i \in [1,n]$ we have with probability $\geq 1-\frac{\delta}{n}$
    \begin{equation*}
        |f_i^0(U) - z_i^v| \leq \frac{1}{\sqrt{m}}\sqrt{2\ln(2n/\delta)} 
    \end{equation*}
    Using Union bound over all samples on the above probabilistic event gives the Lemma \ref{lemma:sub sample}

This result provides an important intuition for our subsequent proof. Specifically, the linear function given by \(f_i^0(W)\) at \(W = U\) can approximate the soft logit \(z_i\) for each \(i\) with high probability when \(m\) is sufficiently large. The complete proof further demonstrates that PGD updates ensure the soft loss, averaged over epochs, remains small even for the non-linear neural network.

\subsection{Convergence Behavior of the Soft Label Surrogate Loss}
Let the average entropy of the soft labels, averaged over all data samples, be given by:
\[
H = \frac{1}{n}\sum_{i=1}^{n}H(p_i),
\]
where 
\[
H(p_i) = -p_i \ln(p_i) -(1-p_i) \ln(1-p_i), \quad \forall i \in [n].
\]
Recall the definition of \(R^{t,KL}(W)\) as:
\[
R^{t,KL}(W) \coloneqq \frac{1}{n} \sum_{i=1}^n \ell^{KL}(p_i, f_i^t(W)).
\] 

The convergence behavior of the soft-label surrogate loss \(R^{KL}(W)\) under projected gradient descent (PGD) with projection radius \(B\) is characterized by the following lemma:

\begin{lemma}\label{lemma:descent soft}
    Let \(\overline{W}\) be any feasible weight matrix for the PGD with projection radius \(B\). That is, it satisfies \(\|\overline{W}_j- W_j(0)\|_2 \leq \frac{B}{\sqrt{m}}, \forall j \in [m]\). Then, the following holds for each iteration \(t\):
    \begin{align}
        \left(2\eta - \eta^2\right) R^{KL}(W(t)) & \leq \|W(t) - \overline{W}\|_{\mathcal{F}}^2 - \|W(t+1) - \overline{W}\|_{\mathcal{F}}^2 + 2\eta R^{t,KL}(\overline{W}) + \eta^2 H.
    \end{align}
    Moreover, it enjoys the following convergence of the surrogate loss function over \(T\) iterations:
    \begin{align}
        \frac{\left(2\eta - \eta^2\right)}{T}\sum_{t < T} R^{KL}(W(t)) & \leq \frac{\|W(0) - \overline{W}\|_{\mathcal{F}}^2 - \|W(T) - \overline{W}\|_{\mathcal{F}}^2}{T} + \frac{2\eta}{T} \sum_{t < T} R^{t,KL}(\overline{W}) + \eta^2 H,
    \end{align}
    using the fact that when $\eta \leq 1$, it implies:
    \begin{align}\label{eq:descent soft}
        \frac{1}{T}\sum_{t < T} R^{KL}(W(t)) & \leq \frac{B^2}{\eta T} + \frac{2}{T}\sum_{t < T} R^{t,KL}(\overline{W}) + \eta H.
    \end{align}
\end{lemma}

\begin{proof}
The proof follows a technique introduced in prior works like \citep{ji2020polylogarithmic,allen2019can} for proving the convergence of gradient descent. The key idea leverages the fact that while the soft-label loss is non-convex with respect to the parameter \(W\), it is convex with respect to the logits \(f_i(W)\).

Let us fix a weight \(\overline{W}\) from the feasible set \(S_B\). Using the steps of PGD, we can write for each \(t \geq 0\):
\begin{align*}
    \|W(t+1) - \overline{W}\|_{\mathcal{F}}^2 
    & \underbrace{\leq}_{(a)} \|\hat{W}(t+1) - \overline{W}\|_{\mathcal{F}}^2 \\
    & \underbrace{=}_{(b)} \|W(t) - \overline{W} - \eta \Delta_W R^{KL}(W(t))\|_{\mathcal{F}}^2 \\
    & = \|W(t) - \overline{W}\|_{\mathcal{F}}^2 - 2\eta \langle \Delta_W R^{KL}(W(t)), W(t) - \overline{W} \rangle \\
    & \quad + \eta^2 \|\Delta_W R^{KL}(W(t))\|_{\mathcal{F}}^2.
\end{align*}

Here step \((a)\) follows from the contraction property of the projection operation in the Frobenius norm. Step \((b)\) uses the weight update rule from the descent step of PGD.

Recall the definition of \(\nabla_{W_j} R^{KL}(W)\), which represents the gradient of the risk with respect to the weight of the \(j^{\text{th}}\) neuron:
\[
\nabla_{W_j} R^{KL}(W) = \frac{1}{n} \sum_{i=1}^n \nabla_{f} \ell^{KL}(p_i, f_i(W)) \nabla_{W_j} f_i(W),
\]
where:
\[
\nabla_{W_j} f_i(W) = \frac{a_j}{\sqrt{m}} \phi_i(W_j) = \frac{a_j}{\sqrt{m}} \indicator(W_j^\top x_i > 0).
\]

Given that \(\|x_i\|_2 \leq 1, \forall i \in [n]\), we can write the following bounds for all \(i \in [n]\):
\[
\|\nabla_{W_j} f_i(W)\|_2 \leq \frac{1}{\sqrt{m}}, \quad \|\nabla_{W} f_i(W)\|_{\mathcal{F}} \leq 1.
\]

The gradient of the loss function \(\ell^{KL}\) with respect to the network output is given by:
\[
\nabla_{f} \ell^{KL}(p_i, f_i(W)) = -p_i \frac{1}{\tau \left(1 + e^{f_i(W)}\right)} + (1 - p_i) \frac{1}{ \left(1 + e^{-f_i(W)}\right)},
\]
which can be rewritten as:
\[
\nabla_{f} \ell^{KL}(p_i, f_i(W)) = \frac{1}{\tau} \left(\frac{1}{1 + e^{-f_i(W)}} - p_i\right).
\]

We can write:
\begin{align*}
   \langle \Delta_W R^{KL}(W(t)) , (W(t)-\overline{W})\rangle & = \frac{1}{n}\sum_{i=1}^{n}\langle \nabla_{f} \ell^{KL}(p_i,f_i(W(t))) \Delta_W f_i(W(t)) , (W(t)-\overline{W})\rangle \\
   & = \frac{1}{n}\sum_{i=1}^{n}\nabla_{f} \ell^{KL}(p_i,f_i(W(t)))\langle  \nabla_W f_i(W(t)) , (W(t)-\overline{W})\rangle \\
   & \underbrace{=}_{(c)} \frac{1}{n}\sum_{i=1}^{n}\nabla_{f} \ell^{KL}(p_i,f_i(W(t)))(f_i(W(t))-f^t_i(\overline{W})\\
   & \underbrace{\geq}_{(d)} R^{KL}(W(t)) - R^{t,KL}(\overline{W})
\end{align*}
Here, \((c)\) follows from the definitions of \(f\) and \(\nabla_{W} f_i(W)\), while \((d)\) uses the convexity of \(\ell^{KL}\) with respect to \(f\).

Using the inequality \(\|\nabla_W f_i(W(t))\|_{\mathcal{F}} \leq 1\) and the subadditivity and absolute homogeneity of the Frobenius norm, we have:
\[
\|\nabla_W R^{KL}(W(t))\|_{\mathcal{F}} \leq \frac{1}{n} \sum_{i=1}^{n} |\nabla_{f} \ell^{KL}(p_i, f_i(W(t)))| \underbrace{\leq}_{(e)} 1,
\]
where \((e)\) follows from the relation:
\[
\nabla_{f} \ell^{KL}(p_i, f_i(W)) =  \frac{1}{1 + e^{-f_i(W)}} - p_i.
\]

Then we can write, 
\begin{align*}
     2\eta R^{KL}(W(t)) & \leq \norm{W(t)-\overline{W}}_{\mathcal{F}}^2 -\norm{W(t+1)-\overline{W}}_{\mathcal{F}}^2 + 2\eta R^{t,KL}(\overline{W}) \\
     & + \eta^2\norm{\Delta_W R^{KL}(W(t))}_{\mathcal{F}}^2
\end{align*}
Now as $\|\nabla_W R^{KL}(W(t))\|_{\mathcal{F}} \leq 1$, 
\begin{align*}
     2\eta R^{KL}(W(t)) & \leq \norm{W(t)-\overline{W}}_{\mathcal{F}}^2 -\norm{W(t+1)-\overline{W}}_{\mathcal{F}}^2 + 2\eta R^{t,KL}(\overline{W}) \\
     & + \eta^2\norm{\Delta_W R^{KL}(W(t))}_{\mathcal{F}}
\end{align*}

Using the identity $\ln(x) \leq 1+x$, the following inequality is derived:
\[
|\nabla_{f} \ell^{KL}(p_i, f_i(W))| \leq \ell^{KL}(p_i, f_i(W)) + H(p_i)
\]
implying,
\[
\|\nabla_W R^{KL}(W(t))\|_{\mathcal{F}} \leq \frac{1}{n} \sum_{i=1}^{n} |\nabla_{f} \ell^{KL}(p_i, f_i(W(t)))| \leq R^{KL}(W(t)) + H.
\]

Putting it all together, we write
\begin{align}
    \left(2\eta-\eta^2\right) R^{KL}(W(t)) & \leq \norm{W(t)-\overline{W}}_{\mathcal{F}}^2 -\norm{W(t+1)-\overline{W}}_{\mathcal{F}}^2 + 2\eta R^{t,kl}(\overline{W})  + \eta^2 H.
\end{align}
Use of a telescoping sum would give us
\begin{align}
    \frac{\left(2\eta-\eta^2\right)}{T}\sum_{t < T} R^{KL}(W(t)) & \leq  \frac{\norm{W(0)-\overline{W}}_{\mathcal{F}}^2 -\norm{W(T)-\overline{W}}_{\mathcal{F}}^2}{T} + \frac{2\eta}{T} \sum_{t < T}R^{t,KL}(\overline{W})  + \eta^2 H \\
    & \leq \frac{B^2}{T} + \frac{2\eta}{T}\sum_{t < T} R^{t,KL}(\overline{W})  + \eta^2 H  
\end{align}
\end{proof}

\subsection{Bounding \texorpdfstring{$|f_i^t(U) - f_i^0(U)|$}{} and \texorpdfstring{$|f_i^t(W(0)) - f_i^0(W(0))|$}{}}
In this subsection, we bound the expressions $|f_i^t(W) - f_i^0(W)|$ and $|f_i^t(W(0)) - f_i^0(W(0))|$ for any $W \in S_B$ and any $t \leq T$. The analysis in this subsection is inspired from the neural tangent kernel literature (see \citep{jacot2020neural,chen2020generalized}). 
\begin{lemma}\label{lemma:NTK approx error}
 Let $W(t)$ be the weight after $t^{\emph{th}}$ iteration of PGD with projection radius $B$. Then, with probability $\geq 1-\delta$, the followings is satisfied for all $t \geq 0$,
\begin{align}
    |f_i^t(W(0))-f_i(W(0))| \leq \frac{1}{\sqrt{m}}\left(\frac{\sqrt{2}}{c \sqrt{\pi}}B^2 + \sqrt{\ln\left(\frac{2n}{\delta}\right)}B \right) , \forall i \in [1,n].
\end{align}
Similarly, when evaluated at any weight $W'$ such that $\|W'_j\|_2 \leq \frac{D}{\sqrt{m}}$ for some $D > 0$, the following holds with probability at least $1-\delta$
\begin{align}
    |f_i^t(W')-f_i^0(W')| \leq \frac{1}{\sqrt{m}}\left(\frac{\sqrt{2}}{c \sqrt{\pi}}BD + \sqrt{\ln\left(\frac{2n}{\delta}\right)}D \right) , \forall i \in [1,n].
\end{align}
\end{lemma}

\begin{proof}
Let us first fix the input data index $i$. Define the set $S_i^t : \{j \in [1,m] : \indicator_{x_i^{\top}W_j(t)\}>0} \neq \indicator_{x_i^{\top}W_j(0)\}>0}\}$. Based on this definition, we can write,
\begin{align*}
    |f_i^t(W)-f_i^0(W)| & = \left| \frac{1}{\sqrt{m}}\sum_{j=1}^m a_j (\indicator_{x_i^{\top}W_j(t) > 0}- \indicator_{x_i^{\top}W_j(0) > 0})x_i^{\top}W_j\right| \\
    & = \left| \frac{1}{\sqrt{m}}\sum_{j \in S_i^t}a_j x_i^{\top}W_j\right|. 
\end{align*}
Now, from the definition of the set $S^t_i$ for all $j \in S^t_i$, we can write the following:
\begin{equation}\label{eq:mismatch_consequence}
    |x_i^{\top}W_j| \leq |x_i^{\top}(W_j-W_j(0))| \leq \frac{B}{\sqrt{m}}.
\end{equation}
Also, from the condition on $W'$ such that $\|W'_j\|_2 \leq \frac{D}{\sqrt{m}}$, as $\norm{x_i}\leq 1, \forall i \in [1,n]$
\begin{align*}
         |x_i^{\top}W_j| & \leq \frac{D}{\sqrt{m}}.
\end{align*}
Hence,
\begin{align*}
    |f_i^t(W(0))-f_i(W(0))| & \leq |S^t_i|\frac{B}{m}.\\
    |f_i^t(U)-f_i^0(U)| & \leq |S^t_i|\frac{D}{m}.
\end{align*}
Next, we give a probabilistic upper bound on the cardinality of set $S^t_i$, $|S^t_i|$ over the symmetric random initialization. Let $sgn(g)$ denote the sign of $g$ for $h \in \mathcal{R}$.  
We can write :
\begin{align*} \label{eq:cardinality_S_t}
    |S^{t}_i| & = \sum_{j=1}^{m/2} \indicator_{sgn(x_i^{\top}W_j(t)\}) \neq sgn(x_i^{\top}W_j(0))}+\sum_{j=m/2}^{m} \indicator_{sgn(x_i^{\top}W_j(t)\}) \neq sgn(x_i^{\top}W_j(0))}.
\end{align*}
From the equation \eqref{eq:mismatch_consequence}, the following holds true for each $j \in [m]$:
\begin{align*}
    \mathbb{P}(sgn(x_i^{\top}W_j) \neq sgn(x_i^{\top}W_j(0))) & \leq \mathbb{P}\left(x_i^{\top}W_j(0) \leq \frac{B}{\sqrt{m}}\right).
\end{align*}

Now, for any $j$, $x_i^{\top}W_j(0)$ is a linear combination of independent random variables drawn from $\mathcal{N}(0,1)$ and, $\norm{x_i}_2 \leq 1$ hence, $x_i^{\top}W_j(0)$ is a normal random variable drawn from $\mathcal{N}(0,\norm{x_i}_2^2)$. Hence, we can write for each $j \in [m]$, 
\begin{align}
    \mathbb{P}\left(x_i^{\top}W_j(0) \leq \frac{B}{\sqrt{m}}\right) \leq \frac{\sqrt{2}B}{c \sqrt{\pi m}}.
\end{align}

Hence, from Hoeffding inequality used on bounded independent random variables, we can write with probability $\geq 1 - \frac{\delta}{2n}$:

\begin{align*}
    \sum_{j=1}^{m/2} \indicator_{sign(x_i^{\top}W_j(t)\}) \neq sign(x_i^{\top}W_j(0))} \leq \frac{m}{2}  \frac{\sqrt{2}B}{c \sqrt{\pi m}} +m\sqrt{\frac{\ln(2n/\delta)}{m}}
\end{align*}

Now because of the symmetric random initialization, using the union bound on probabilities on the expression in equation \eqref{eq:cardinality_S_t}, we can write: with probability $\geq 1-\frac{\delta}{n}$,
\begin{align*}
    |S^t_i| \leq   \frac{\sqrt{2}Bm}{c \sqrt{\pi m}} + 2m\sqrt{\frac{\ln(2n/\delta)}{m}}
\end{align*}

Putting it all together : for each $i \in [1,n]$, with probability $\geq 1-\frac{\delta}{n}$
\begin{align*}
|f_i^t(W(0))-f_i(W(0))| & \leq \left(\frac{\sqrt{2}Bm}{c \sqrt{\pi m}} + 2m\sqrt{\frac{\ln(2n/\delta)}{m}}\right)\frac{B}{m} \\
& = \frac{1}{\sqrt{m}}\left(\frac{\sqrt{2}}{c \sqrt{\pi}}B^2 + 2\sqrt{\ln\left(\frac{2n}{\delta}\right)}B \right)\\
    |f_i^t(U)-f_i^0(U)| & \leq \left(\frac{\sqrt{2}Bm}{c \sqrt{\pi m}} + 2m\sqrt{\frac{\ln(2n/\delta)}{m}}\right)\frac{D}{m}\\
    & = \frac{1}{\sqrt{m}}\left(\frac{\sqrt{2}}{c \sqrt{\pi}}BD + 2\sqrt{\ln\left(\frac{2n}{\delta}\right)}D \right)
\end{align*}
Now using the union bound on the above probability events over all example index $i \in [1,n]$, we get the Lemma \ref{lemma:NTK approx error}.

\end{proof}

Now, putting $W' = U$ in the Lemma \ref{lemma:NTK approx error}, we get the statement:
the following holds with probability at least $1-\delta$
\begin{align}
    |f_i^t(U)-f_i^0(U)| \leq \frac{1}{\sqrt{m}}\left(\frac{\sqrt{2}}{c \sqrt{\pi}}B + \sqrt{\ln\left(\frac{2n}{\delta}\right)} \right) , \forall i \in [1,n].
\end{align}
\subsection{Final Steps of the Proof of Theorem \ref{thm:soft label kl}}

From Lemma \ref{lemma:descent soft}, we have the following convergence guarantee of the Projected Gradient Descent(PGD) for the surrogate loss for soft label training:
\begin{equation}
    \frac{1}{T}\sum_{t < T} R^{KL}(W(t)) \leq \frac{B^2}{\eta T} + \frac{2}{T}\sum_{t < T} R^{t,kl}(\overline{W})  + \eta H 
\end{equation}

 We chose the weight $\overline{W}$ as $\overline{W} = W(0)+U$ where $U_j = \frac{a_j}{\sqrt{m}}v(W_j(0))$. Under the sufficient conditions described next, we achieve the target surrogate loss $\frac{1}{T}\sum_{t < T}R^{KL}(W(t)) \leq \beta$ for some $\beta > 0$.
\begin{enumerate}
    \item 
    \begin{equation}
        \eta \leq \frac{\beta}{3H}.
    \end{equation}
    \item 
    \begin{equation}
        \frac{1}{\eta T} \leq \frac{\epsilon}{3}.
    \end{equation}
    This is satisfied when : 
    \begin{equation}
        T \geq \frac{9HB^2}{\beta^2 }.
    \end{equation}
    \item 
    \begin{equation}
        \frac{2}{T}\sum_{t\leq T}R^{t,KL}(\overline{W}) \leq \frac{\beta}{3}.
    \end{equation}
    This is satisfied when 
    $\ell^{KL}(p_i,\mu(f_i^t(\overline{W}))) \leq \frac{\beta}{6}$ for each $i \in [1,n]$.
\end{enumerate}

Let us first fix the input sample index $i$. We focus to upper-bound $\ell^{KL}(p_i,\mu(f_i^t(\overline{W})))$ for the correctly chosen $\overline{W}$ for each $i \in [n]$. For this purpose, we use the following version of the reverse Pinsker's inequality from Lemma 4.1 of \cite{G_tze_2019}: for $p$ and $q$ being parameters of different Bernoulli distribution with $0< p,q < 1$
\begin{equation*}
    2|p-q|^2 \leq D_{KL}(p||q) \leq \frac{2}{\min(q,1-q)}|p-q|^2.
\end{equation*}

Substituting $p = p_i$ and $q = \mu(f_i^t(\overline{W}))$, we get
\begin{align*}
    \ell^{KL}(p_i,f_i^t(\overline{W})) & \leq \frac{2}{\min(\mu(f_i^t(\overline{W})),1-\mu(f_i^t(\overline{W})))}|p_i-\mu(f_i^t(\overline{W}))|^2.
\end{align*}

The sigmoid function $\mu(g) = \frac{1}{1+e^{-g}}$ is a 1-Lipschitz function, we can write,
\begin{equation*}
    |p_i-\mu(f_i^t(\overline{W}))| = |\mu(z_i)-\mu(f_i^t(\overline{W}))| \leq |z_i - f_i^t(\overline{W})|.
\end{equation*}

Also, 
\begin{align*}
    \frac{2}{\min(\mu(f_i^t(\overline{W})),1-\mu(f_i^t(\overline{W})))} & = \frac{2}{\min\left(\frac{1}{1+e^{-f_i^t(\overline{W})}},\frac{1}{1+e^{f_i^t(\overline{W})}}\right)} \\
    & = 2(1+e^{|f_i^t(\overline{W})|}).
\end{align*}

Putting it together, we can write,
\begin{equation}\label{eq:soft label kl loss}
    l^{kl}(p_i,f_i^t(W))  \leq 2(1+e^{|f^t_i(W)|})|z_i-f_i^t(W)|^2
\end{equation}

Now from Lemma \ref{lemma:sub sample},  we have with probability $\geq 1-\delta$
    \begin{equation*}
        |f_i^0(U) - z_i| \leq \frac{1}{\sqrt{m}}\sqrt{2\ln(2n/\delta)} ; \forall i \in [1,n].
    \end{equation*}
    
Using triangle inequality, and using $\overline{W} = W(0)+U$, we have,
\begin{equation*}
    |z_i-f_i^t(\overline{W})| \leq |f_i^0(U) - z_i| + |f_i^0(U)-f_i^t(U)| + |f_i^0(W(0))-f_i^t(W(0))|.
\end{equation*}

Next, we use the Lemma \ref{lemma:NTK approx error} to bound the expression $|f_i^0(U)-f_i^t(U)| + |f_i^0(W(0))-f_i^t(W(0))|$. From Lemma \ref{lemma:NTK approx error} (putting $W'=U$), using a union bound, we can write with probability at least $1-2\delta$, simultaneously for all $i \in [1,n]$. 
\begin{align*}
    |f_i^0(U)-f_i^t(U)| + |f_i^0(W(0))-f_i^t(W(0))| 
   \leq  \frac{2}{\sqrt{m}}\left(\frac{\sqrt{2}}{c \sqrt{\pi}}B^2 + 2\sqrt{\ln\left(\frac{2n}{\delta}\right)}B \right) 
\end{align*}
Here we assume $B \geq 1$. In Theorem \ref{thm:soft label kl}, we set $B = 1$. This gives us

\begin{equation*}
    |f_i^t(\overline{W})-z_i| \leq \frac{2}{\sqrt{m}}\left(\frac{\sqrt{2}}{c \sqrt{\pi}}B^2 + (2B+1)\sqrt{\ln\left(\frac{2n}{\delta}\right)}\right)
\end{equation*}
 
 We upper-bound the term $(1+e^{|f^t_i(W)|})$. Notice the following upper-bound on $f_i^0(\overline{W})$.
\begin{align}
    |f_i^0(\overline{W})|& \underbrace{\leq}_{(f)}  \left| \frac{1}{\sqrt{m}}\sum_{j=1}^m a_j \indicator_{x_i^{\top}W_j(0) > 0}x_i^{\top}U_j\right|\\
    & \underbrace{=}_{(g)} \left| \frac{1}{m}\sum_{j=1}^m \indicator_{x_i^{\top}W_j(0) > 0}x_i^{\top}v(W(0)_j)\right|  \\
    & \underbrace{\leq}_{(h)} 1  
\end{align}
where, $(f)$ is using the fact that $f_i(W(0))= 0 ,\forall i \in [1,n]$ , $(g)$ is from the definition of $U$, and, $(h)$ uses $\norm{x_i}_2\leq 1$ and $\norm{v(W(0)_j)}_2 \leq 1$ . This implies, $(1+e^{|f^t_i(W)|}) \leq 1+e^2$.

Recall from Equation \ref{eq:soft label kl loss},
\begin{equation}
    l^{kl}(p_i,f_i^t(W))  \leq 2\underbrace{(1+e^{|f^t_i(W)|})}_{I}\underbrace{|z_i-f_i^t(W)|^2}_{II}
\end{equation}

 Putting the above arguments together, under the Following sufficient condition:
        \begin{equation}\label{eq:condition_soft_m_proof}
            m \geq \frac{96}{(1+e^2)\beta}\left(\sqrt{\frac{2}{\pi}}\frac{B^2}{c} + (2B+1)\sqrt{\ln\left(\frac{2n}{\delta}\right)} \right)^2
        \end{equation}
   
    We can write 
    \begin{enumerate}
        \item $I \leq 1+e^2$
        \item $II \leq \frac{\beta}{6}$
    \end{enumerate}

Finally adding the conditions
   \begin{equation*}
        \eta \leq \frac{\beta}{3H}.
    \end{equation*} 
    \begin{equation*}
        T \geq \frac{9HB^2}{\beta^2 }.
    \end{equation*}
along with putting $B = 1$ in condition \eqref{eq:condition_soft_m_proof} gives us the Theorem \ref{thm:soft label kl}.
\section{Proof of Corollary \ref{cor: soft label requirement indicator}} \label{appendix:proof of cor:soft label}
The Corollary follows from the Theorem \ref{thm:soft label kl} and the Lemma \ref{lemma:relating indicator to kl}. Let us first give a proof of Lemma \ref{lemma:relating indicator to kl}.

\subsection{Proof of Lemma \ref{lemma:relating indicator to kl}}
First, we use the Pinsker's inequality to upper-bound the classification loss $R(W)$ in terms of the surrogate loss $R^{KL}(W)$ in Lemma \ref{lemma:relating indicator to kl appendix} for any soft labels $p'_i$ satisfying $y_i \nu^{-1}(p_i') > 0$, $\forall i \in [1,n]$.
\begin{lemma}\label{lemma:relating indicator to kl appendix}
    If we assume $0< p_i' < 1$, and $y_i \nu^{-1}(p_i') > 0$, $\forall i \in [1,n]$\footnote{This assumption ensures the soft labels are correctly indicating the true class for inputs in context}, we can relate the classification loss $R(W)$ to the surrogate loss $R^{KL}(W)$ as:
    \begin{equation}
        R(W(t)) \leq \frac{1}{2\min_i\left|p'_i-\frac{1}{2}\right|^2}R^{kl}(W(t))
    \end{equation}
\end{lemma}
\begin{proof}
    We use the Pinsker's inequality :
    \begin{equation}
        KL(p'_i||\mu(f_i(W(t)))) \geq 2 |p'_i-\mu(f_i(W(t)))|^2
    \end{equation}
    The followings are true when $y_i \mu_{\tau}^{-1}(p'_i) > 0, \forall i$:\\
   \textbf{Case 1: $y_i = +1$} When there is a misclassification by the neural network, : $p'_i \geq \frac{1}{2}$ and $\mu_{\tau}(f_i(W(t))) \leq \frac{1}{2}$, Hence, $|p'_i-\mu_{\tau}(f_i(W(t)))| \geq \left|p'_i-\frac{1}{2}\right|$\\
   \textbf{Case 2: $y_i = -1$} When there is a misclassification by the neural network,: $p'_i \leq \frac{1}{2}$ and $\mu_{\tau}(f_i(W(t))) \geq \frac{1}{2}$, Hence, $|p'_i-\mu_{\tau}(f_i(W(t)))| \geq \left|p'_i-\frac{1}{2}\right|$.
\end{proof}

Now, the soft label used in the paper $p_i$ and corresponding logit $z_i$ for each $i$, satisfies $|z_i| \geq \gamma$. Hence, for each $i \in [n]$, $|p_i-\frac{1}{2}| \geq \frac{1}{2}\frac{e^|z_i|-1}{e^|z_i|-1} \geq  \frac{1}{2}\frac{e^{\gamma}-1}{e^{\gamma}-1}$. Now as $\gamma \leq 1$ by our construction of $z_i$, this implies, $ \frac{1}{2\min_i\left|p'_i-\frac{1}{2}\right|^2} \leq \frac{32}{\gamma^2}$ for all $i \in [n]$. Hence, we can write,

\begin{equation*}
    R(W) \leq \frac{32}{\gamma^2}R^{KL}(W).
\end{equation*}

\subsection{Final Steps for Proving Corollary \ref{cor: soft label requirement indicator}}
From Lemma \ref{lemma:relating indicator to kl}, to achieve the target classification loss $R(W(t)$ averaged over $t < T$ to be smaller than some $\epsilon > 0$, the following condition is sufficient :
\begin{equation*}
    \frac{1}{T}\sum_{t < T}R^{KL}(W(t)) \leq \frac{32 \epsilon}{\gamma}
\end{equation*}
Now substituting $\beta$ with $\frac{32 \epsilon}{\gamma}$ in the proof of the Theorem \ref{thm:soft label kl}, we arrive at the result by Corollary \ref{cor: soft label requirement indicator}.

\section{Sufficient Condition for Hard Label Training with Projected Gradient Descent: Proof of Proposition \ref{prop: hard label PGD}} \label{appendix:hard label}
We adapted the analysis of the paper \citep{ji2020polylogarithmic} to the PGD setting to arrive at the sufficient conditions for the classification loss averaged over epoch to be smaller than some $\epsilon > 0$ using hard label training. 

Similar to the soft label training case, let use define the following quantity corresponding to the hard label training:
\begin{equation}
     R^{t,h}(W) \coloneqq \frac{1}{n}\sum_{i=1}^n \ln(1+\exp(-y_if_i^t(W))).
\end{equation}
Lemma 2.6 in \citep{ji2020polylogarithmic} describing the convergence behavior of the hard label surrogate loss in the gradient descent training can be adapted to the PGD setting considered in this paper for hard label training using a projection to the set $S_B$ for some $B > 0$ at each step of the weight update. Let us assume. The use of the contraction property of the projection operator in the $l_2$ norm sense shows that the same guarantee holds for PGD case. 
Hence, for the PGD with projection radius $B$, we have for the hard label training when $\eta \leq 1$, 
\begin{equation*}
    \frac{1}{T}\sum_{t < T} R^{h}(W(t))    \leq \frac{B^2}{\eta T} + \frac{2\eta}{T}\sum_{t < T} R^{t,h}(\overline{W}).
\end{equation*}

 Now assign $\overline{W} = W(0)+BU$.
 
 Recall from the proof of Lemma \ref{lemma:NTK approx error}, with probability $\geq 1-\delta$ simultaneously for all $i \in [1,n]$ the following holds:
\begin{align*}
    |f_i^t(\overline{W})-f_i^0(\overline{W})|\leq  \frac{2}{\sqrt{m}}\left(\frac{\sqrt{2}}{c \sqrt{\pi}}B^2 + 2\sqrt{\ln\left(\frac{2n}{\delta}\right)}B \right).
\end{align*}

Also, recall from lemma \ref{lemma:sub sample},  we have with probability $\geq 1-\delta$
    \begin{equation}
        |f_i^0(U) - z_i| \leq \frac{1}{\sqrt{m}}\sqrt{2\ln(2n/\delta)} ; \forall i \in [1,n]
    \end{equation}
which implies $y_if_i^0(U) \geq \gamma - \frac{1}{\sqrt{m}}\sqrt{2\ln(2n/\delta)} ; \forall i \in [1,n]$.
Using triangle inequality, we can write:
\begin{equation*}
    y_i f_i^t(\overline{W})  \geq y_if_i^0(U) -  |f_i^t(\overline{W})-f_i^0(\overline{W})|.
\end{equation*}
Putting it together, we can write with probability $\geq 1-2\delta$,
simultaneously for all $i \in [1,n]$, the following holds true
\begin{align*}
    y_i f_i^t(\overline{W})  \geq B\gamma- \frac{2}{\sqrt{m}}\left(\frac{\sqrt{2}}{c \sqrt{\pi}}B^2 + 3B\sqrt{\ln\left(\frac{2n}{\delta}\right)} \right).
\end{align*}

We know the relation between $R(W)$ and $R^h(W)$ is given by $ R(W) \leq \frac{1}{\ln(2)}R^{ce}(W)$. Hence, we can write the sufficient condition for making $\frac{1}{T}\sum_{t<T}R(W(t)) \leq \epsilon$ can be written as

\begin{enumerate}
    \item \begin{equation}
        B = \frac{2}{\gamma}\ln \left(\frac{2}{\ln(2)\epsilon }\right).
    \end{equation}
    \item \begin{equation}
        m \geq \frac{16}{\gamma^4}\left(\frac{2\sqrt{2}}{c\sqrt{\pi}}ln\left(\frac{2}{\ln(2)\epsilon}\right)+3\sqrt{\ln\left(\frac{2n}{\delta}\right)}\right)^2.
    \end{equation}
    \item \begin{equation}
        T \geq \frac{8}{\gamma^2 \eta}\frac{\ln^2 \left(\frac{2}{\ln(2)\epsilon }\right)}{\ln(2)\epsilon}.
    \end{equation}
\end{enumerate}

\section{Intuition Behind the Proof Technique Resulting Better Neuron Requirement Using Soft Labels}

Under Assumption~\ref{ass:RKHS}, the random initialization of the neural network acts as a feature map that separates the dataset with a margin of order \(O\left(\gamma\right)\), provided the number of neurons is of the order \(O\left(\frac{1}{\gamma^2}\right)\). This is formalized in the Lemma \ref{lemma:sub sample}.

In other words, by choosing \(m \approx 1/\gamma^2\), the weight \(U\) should be able to separate the data. Both soft-label training and hard-label training aims to find a neural network weight that is in the direction of \(U\). However, the magnitudes of these weights differ significantly. For soft-label training, only a small scaler multiple of $U$ is needed to match the logits \(z_i\) for each input \(x_i\). In contrast, for hard-label training, the neural network's output after the softmax layer must match exactly \(0\) or \(1\). This stricter requirement needs a much larger multiple of $U$. To account for this greater deviation of weights from the initialization, hard-label training requires a substantially higher number of neurons to preserve the feature space separability effectively. As a result, the neuron requirement increases to \(\frac{1}{\gamma^4}\), which is much higher than \(\frac{1}{\gamma^2}\) as predicted by Lemma 7.

\end{document}